\documentclass[sigconf]{acmart}
\renewcommand\footnotetextcopyrightpermission[1]{} % removes footnote with conference information in first column
\usepackage{bm}
\usepackage{booktabs}
\usepackage{multirow}

\usepackage{setspace}
\usepackage{url}

\usepackage{graphicx}
\usepackage{hyperref}

\newcommand*{\prob}{\mathsf{P}}

\acmConference[arXiv]{}{December}{2020}

\begin{document}

\title[Split: Disentangling Brand-Customer Interactions]{Split: Inferring Unobserved Event Probabilities for Disentangling Brand-Customer Interactions}

\author{Ayush Chauhan}
\affiliation{
    \institution{Adobe Research}
    \city{Bangalore}
    \country{India}
    }
\email{ayuchauh@adobe.com}

\author{Aditya Anand}
\affiliation{
    \institution{IIT Kharagpur}
    \city{Kharagpur}
    \country{India}
    }
\email{adityaanand1998g@gmail.com}
\authornote{The author was at Adobe Research at the time of this work.}

\author{Shaddy Garg}
\affiliation{
    \institution{Adobe Research}
    \city{Bangalore}
    \country{India}
    }
\email{shadgarg@adobe.com}

\author{Sunny Dhamnani}
\affiliation{
    \institution{Georgia Institute of Technology}
    \city{Atlanta}
    \country{USA}
    }
\email{dhamnanisunny.1402@gmail.com}
\authornotemark[1]

\author{Shiv Kumar Saini}
\affiliation{
    \institution{Adobe Research}
    \city{Bangalore}
    \country{India}
    }
\email{shsaini@adobe.com}

\renewcommand{\shortauthors}{Ayush, et al.}

\begin{abstract}
Often, data contains only composite events composed of multiple events, some observed and some unobserved. 
For example, search ad click is observed by a brand, whereas which customers were shown a search ad - an actionable variable - is often not observed.
%\shiv{TODO: CAN WE GET A CITATION FOR ABOVE IN INTRODUCTION}
%One example is the case when the goal is to estimate the effect of showing a search ad when only a search click is observed. A search click is observed when an individual searches, a brand shows an ad, and then the individual clicks on the ad.
In such cases, inference is not possible on unobserved event. 
This occurs when a marketing action is taken over earned and paid digital channels.
Similar setting arises in numerous datasets where multiple actors interact. 
One approach is to use the composite event as a proxy for the unobserved event of interest.
However, this leads to invalid inference.
This paper takes a direct approach whereby an event of interest is identified based on information on the composite event and aggregate data on composite events (e.g. total number of search ads shown).
This work contributes to the literature by proving identification of the unobserved events' probabilities up to a scalar factor under mild condition.
We propose an approach to identify the scalar factor by using aggregate data that is usually available from earned and paid channels.
The factor is identified by adding a loss term to the usual cross-entropy loss.
We validate the approach on three synthetic datasets. 
In addition, the approach is validated on a real marketing problem where some observed events are hidden from the algorithm for validation. 
The proposed modification to the cross-entropy loss function improves the average performance by 46\%.

\end{abstract}


\maketitle

\section{Introduction}
%Points to be made
%% 1. Example from marketing channel
%% 2. This is a general problem
%% 3. Where else we can find the examples where such problem will exist - 
%need a few examples
% To do
%% Show 

%This paper considers the general setting where data contains only composite events which are composed of multiple events, some observed and some unobserved. 
%In such cases, inference is not possible on unobserved event.
%This paper studies the conditions conditions under which the unobserved events can be identified. 
%This setting occurs in a variety of datasets and most common in cases where data is generated across multiple platform.
% other examples
%     Settings in which this problem is faced - Cite - Attribution papers, Causal Inference papers, Criteo challenge

A brand interacts with a customer on owned channels, for example own website, as well as earned and paid channels. 
The main focus of marketing on earned and paid channel is to bring customers to owned channels. 
Some of the actions a marketer takes on earned and paid channels are bidding for a search, display or a social media ad, or send an email through email marketing vendors. 
Invariably, the marketer observes the result of a marketing action only when the subjects of the marketing action visit owned channels. 
The information on who saw an ad is difficult to get for some channels (search, display) and impossible for others (social media platforms)\cite{Bianchi2018}. 
Often, the marketer has access to aggregate data on how many customers saw the ad. 
However, in absence of targeting data, it is not possible for a marketer to make data-driven decisions which results in sub-optimal targeting.

The goal of this paper is to provide theoretical and empirical identification approaches for conducting inference on unobserved individual-level targeting data using only composite events such as search ad click and aggregated data on targeting. Note that the behavioral and demographic data are available for each individual but whether someone was shown an ad or not is not available. In past, this problem has been dealt with by using the composite event as a proxy for the event of interest \cite{ji2017additional}. 
This leads to over-estimation of the effect of a marketer action. 
\cite{besbes2019shapley} uses a latent state model to estimate the effect of a search ad shown. 
We take a direct approach of identifying unobserved event(s) when only a composite event is observed.
Note that, if marketer actions are observed, this is a standard causal inference problem that can be solved by estimating the individual treatment effect from observational data using techniques such as \cite{Shalit2017EstimatingIT}, \cite{Saini2019MTE}.

In this paper, the starting point is a setting where the observed composite event is an interaction of two (or more) unobserved events.
%OLD The main intuition behind the solution is based on exploiting independent sources of variation in data that affects the two unobserved events.
The solution is based on the assumption that there exist two mutually independent sources of variation in data such that one such source affects only one of the two unobserved event and the second source affects only the other unobserved event. 
%Note that, we do not need to identify these independent sources of variation. Mere existence is sufficient.
In a dataset with a large number of variables, such an assumption is not restrictive because we can easily find variables that affect one and not the other event and vice versa.

We prove identification of the unobserved event probabilities up to a scale factor under mild conditions. 
The scale factor can be identified if the probability of the unobserved events is arbitrarily close to one for a subset of customers.
This is a restrictive assumption because we are unlikely to observe high probability of the events that are of interest to us.
To overcome the restrictive assumption, we use additional data that is easily available.
In these settings, often aggregate level targeting data is available.
In particular, we use fraction of eligible customers who were shown a particular ad as the additional data. 
These aggregate level data are readily available through Demand Side Platforms as well as social media platforms \cite{fb2020}.
We propose a novel solution wherein the aggregate data is used in conjunction with individual customer level data to achieve identification of the unobserved probabilities without using the restrictive assumption. 

The approach is validated on multiple synthetic datasets created by changing the settings of the identification assumptions. 
In addition, we apply the approach on two real customer interaction datasets. 
The first dataset is an email interaction dataset where we observe email send, open, click.
We hide email-send information from the algorithm and use it for validating the approach.
The second customer interaction dataset is based on paid and organic search clicks.
In this case, the observed events are complex functions of observed and unobserved events' probabilities. 
We do not observe which customer is served a search ad.
Hence, the Search data cannot be used directly for validating the approach.
However, we can test consistency of the results across different runs with different starting value. 
Our hypothesis is that when the unobserved probabilities are not identified, different initialization will result in different estimates of the unobserved probabilities. 

In this paper we make four contributions. 
First, we prove that the unobserved events' probabilities are identified up to a scale under reasonable conditions.
Second, we propose an additional aggregate loss term in the cross-entropy loss function that helps in identification of the scale. 
The aggregate loss term enforces a soft constraint based on aggregate targeting data on a model to estimate individual level probabilities. 
Third, the approach is validated on synthetic data. The proposed approach shows 36\% to 44\% reduction in error on an average, as measured by MSE and MAPE respectively, with respect to a baseline approach for the most realistic setting. For other settings, the results are even better.
Fourth, the approach is applied on an email marketing dataset and a search ad dataset. The proposed approach reduces the MSE by 69\% for unobserved probabilities of email-send and by 14\% for the unobserved probabilities of email-open given email-sent. 
We show how to apply the approach to a search ad problem where ad targeting data is not available. 
%We validate the approach on this dataset by comparing consistency across runs. The proposed approach improves consistency of results, as measured by fraction of likewise classified samples, \ayush{by 1\% on average.}

The next section puts this work in the context of previous work.

\vspace{-0.5em}
\section{Related Work}
%Here is the papers doing attribution: \cite{Shalit2017EstimatingIT}.
This paper provides an approach for identifying probabilities of unobserved events when intersection of the events is observed. We did not find any papers that directly address the problem. However, the problem studied here arises in many settings. In this section, the approaches that are followed in these settings are discussed.

The problem of unobserved events and observed composite events occurs in marketing attribution. 
Primarily, two approaches are employed.
The first one is where the composite event is taken as a proxy. This essentially means ignoring the problem. There are a large number of papers that follow this approach. Most of these attribute a success event to marketing touches \cite{ji2017additional}, \cite{zhao2019revenue}, \cite{Ren2018DualAttention}. 
% Todo: Describe these papers and what is the problem
The second approach is where it is assumed that a customer makes decision based on her latent states \cite{besbes2019shapley}. In these papers, a latent state model is estimated to control for the bias introduced by the missing events. 
There are other papers that use indirect approaches to solve the missing data problem \cite{danaher2018delusion}, \cite{abhishek2017multi}, \cite{kakalejvcik2018multichannel}.
The fraction of conversions tends to be higher for the marketing touch points which are likely to occur during the later stages of the purchase cycles (e.g. search ad click) as compared to the touch points early in the purchase cycle (e.g. email).
A naive approach will lead to over-attribution of a success event to marketing channels such as search ad.
The goal of these papers is to correct for the biased attribution of a success event.
The typical analysis is done by incorporating touches across multiple marketing channels.
If the goal is to study the effect of a single marketing channel, such as search ad, these approaches are not applicable.
%\shiv{TODO: Why we do not include any of the above in baseline}.
%\shaddy{Another area where missing data problem is explored are in papers like \cite{bertsimas2017predictive}, where the approach of imputation of missing data is considered as an optimization task. This literature uses different cost functions derived from KNN, SVM and Optimal Decision Tree models. \cite{smieja2018processing} model the uncertainity on missing attributes by their probability density functions learned from the activations of the neural network and then use this information to represent the missing data. Other literature in this area mainly talks about imputation using KNN and modified regression techniques. This, however is a completely different problem from what we are trying to solve because our problem deals with the case of missing events rather than some missing values of an observed event.}

Another area where missing events needs to be accounted for is in causal inference literature. However, the goal of this literature is not to do inference on the unobserved events but to use techniques to control for the unobserved confounders so that the inference on the observed data is valid. \cite{bertsimas2017predictive} and \cite{smieja2018processing} discuss another area where missing data problem arises. They talk about imputation or modelling the missing data using KNN, modified regression, and other techniques. Ours is a different problem because we deal with the case of missing events rather than some missing values of an observed event.

There is a large literature on inferring individual-level behavior using aggregate data. 
This research area is called Ecological Inference \cite{Schuessler10578}. 
The techniques studied in this area are not directly applicable because we \emph{have} individual level data as well as labels for at least some related events, but no labels for the events of interest.
We are not able to find any technique which is directly applicable for the problem at hand. 

In absence of baselines from the past literature, we will validate our approach by providing a proof and conducting experiments with and without proposed modifications to the cross-entrophy loss function.

\vspace{-0.5em}
% Notation: \mathbf{}, with capital letter for random vectors. 
\section{Model} 
\label{sec:model}
The main goal of this paper is identification of the unobserved events from  related observed events. The events observed in the real-world data can be simple or complex functions of the unobserved events of interest. Here, we first prove identifiability of the unobserved events in case of composite functions that one is likely to encounter in real world. We begin with a simple product of two unobserved functions and then extend it to other settings. Next, we present a solution strategy, based on the proof, to estimate the functions in practice.
\vspace{-0.5em}
%Let $f_n(\cdot)$, $g_n(\cdot)$, and $h_n(\cdot)$ be the estimates of $f(\cdot)$, $g(\cdot)$, and $h(\cdot)$, respectively.
\subsection{Identification}
\label{subsec:identifiability}
Let us begin with a simple example where the observed event is a product of two unobserved events:

\textit{Example 1:} Let $P_{a,c} = \prob(\texttt{Click} = 1, \texttt{Ad Shown} = 1 | \mathbf{X})$, $P_{c} = \prob(\texttt{Click} = 1 | \texttt{Ad Shown} = 1, \mathbf{X})$, and $P_{a} = \prob(\texttt{Ad Shown} = 1 | \mathbf{X})$, where $\mathbf{X}$ is a vector of exogenous or pre-determined random variables.
Then, $P_{a,c} = P_{c} * P_{a}$. The goal is to find out whether $P_{c}$ and $P_{a}$ are identifiable under the following set-up:
\begin{enumerate}
	\item The event corresponding to $P_{a,c}$ and the vector $\mathbf{X}$ are observed in data
	\item The events corresponding to $P_{c}$ and $P_{a}$ are not observed
\end{enumerate}
Let $\mathbf{X_1}, \mathbf{X_2} \subseteq \mathbf{X}$ are observed random vectors that determine $P_{c}$ and $P_{a}$ respectively. As before, the product of the two in turn determines $P_{a,c}$. Let $P_{c} = f(\mathbf{X_1})$ and $P_{a} = g(\mathbf{X_2})$. The composite event,
\begin{equation} \label{eq:prodprob}
P_{a,c} = f(\mathbf{X_1})*g(\mathbf{X_2}) = h(\mathbf{X_1},\mathbf{X_2})    
\end{equation}
Since the event corresponding to the composite event is observed in data, $h(\mathbf{X_1}, \mathbf{X_2})$ can be easily estimated using Supervised Learning approaches. Assuming that there is non-zero probability of the composite event taking place $\forall \mathbf{X_1}, \mathbf{X_2}$ and since $f(\cdot), g(\cdot)$ represent probability functions of the unobserved events, we have:
\begin{equation}
\label{eq:prob-constraint}
    f(\mathbf{X_1}), g(\mathbf{X_2}) \in (0, 1] \ \ \forall\ \ \mathbf{X_1}, \mathbf{X_2} \text{ in } supp(\mathbf{X_1}), supp(\mathbf{X_2}) \text{ resp.}
\end{equation}
where $supp(X)$ is the support of X. Next, we study the assumptions needed to identify $f(\cdot)$, $g(\cdot)$, given that the probability of the composite event is identified.
\begin{enumerate}
    \item[\textbf{A.1}] $\mathbf{X_1}$ and $\mathbf{X_2}$ have non-zero independent components which determine $f(\cdot)$ and $g(\cdot)$ respectively, i.e. $\mathbf{X_1}=\{\mathbf{u},\mathbf{v}\}$ and $\mathbf{X_2}=\{\mathbf{v},\mathbf{w}\}$ such that $\mathbf{u}, \mathbf{v}$ and $\mathbf{w}$ are pair-wise independent, and $\mathbf{u}\neq 0, \mathbf{w}\neq 0$ and partial derivative of $f(\cdot), g(\cdot)$ is non-zero with respect to $\mathbf{u}, \mathbf{w}$ respectively.
    \item[\textbf{A.2}] The set of covariates that determine $f(\cdot)$ and $g(\cdot)$ are known a-priori to the modeler, i.e. the break-up of $\mathbf{X}$ into $\mathbf{X_1}$ and $\mathbf{X_2}$ is known. 
\end{enumerate}
It is worth noting that instead of \textbf{A.1}, a weaker assumption of strong decomposability between $\mathbf{X_1}$ and $\mathbf{X_2}$ is sufficient for identifiability. We make the stronger assumption for better exposition.
\begin{theorem} \label{theorem:identuptoscale}
Given \textbf{A.1} and \textbf{A.2},
\begin{enumerate}
    \item[a)] $f(\mathbf{X_1})$ and $g(\mathbf{X_2})$ are identifiable up to a scale, and
    \item[b)] the unidentified scale parameter is a constant.
\end{enumerate} 
\end{theorem}
\begin{proof}
From \textbf{A.2}, one can model $f(\cdot)$ and $g(\cdot)$ as non-zero probability functions of $\mathbf{X_1}$ and $\mathbf{X_2}$ respectively. Let the estimated functions be $f^{'}(\mathbf{X_1})$ and $g^{'}(\mathbf{X_2})$. Since the composite event of their product is observed and given the large size of data with no sampling, the composite event probability can be estimated accurately using a flexible estimation approach. Therefore, we have $f^{'}(\mathbf{X_1}) * g^{'}(\mathbf{X_2}) = h(\mathbf{X_1}, \mathbf{X_2}) = f(\mathbf{X_1}) * g(\mathbf{X_2})$.
% \begin{equation}
%     f^{'}(\mathbf{X_1}) * g^{'}(\mathbf{X_2}) = h(\mathbf{X_1}, \mathbf{X_2}) = f(\mathbf{X_1}) * g(\mathbf{X_2}).
% \end{equation}  
Using Equation \ref{eq:prob-constraint} and rearranging, we get
\begin{equation} \label{eq:equality}
\frac{f^{'}(\mathbf{X_1})}{f(\mathbf{X_1})} = \left( \frac{g^{'}(\mathbf{X_2})}{g(\mathbf{X_2})} \right)^{-1}. 
\end{equation}
From \textbf{A.1}, the LHS and RHS in equation \ref{eq:equality} are functions of independent sets of variables, $\mathbf{u}$ and $\mathbf{w}$. This implies that $\mathbf{X_1}$ can be varied independent of $\mathbf{X_2}$ by varying $\mathbf{u}$ but keeping $\mathbf{v}$ and $\mathbf{w}$ fixed. Since $f(\cdot)$ is a non-constant function of $\mathbf{u}$ and the RHS does not change, the ratio in the LHS must be constant for all values of $\mathbf{u}$. Since this is true for all values of $\mathbf{v}$ and $\mathbf{w}$, the estimated functions are a constant multiple of the true functions for all values of $\mathbf{X_1}$ and $\mathbf{X_2}$.
\end{proof}

Next, we state another assumption needed to identify the scale parameter. Although this assumption can be verified, it is rarely satisfied in real data. This fact is later used to illustrate the need for an alternate strategy to identify the scale of estimates of $f(\cdot)$ and $g(\cdot)$.
\begin{enumerate}
    \item[\textbf{A.3}] We have $f(\cdot) = 1$ and $g(\cdot) = 1$ for at least some customers. It is a very restrictive assumption, in that one is highly unlikely to observe such high probability for the events of interest. 
\end{enumerate}

\begin{corollary} \label{corollary:scaleident}
Given \textbf{A.1}, \textbf{A.2} and \textbf{A.3}, the scale parameter is equal to 1 and the unobserved events are exactly identified.
\end{corollary}
\begin{proof}
From Theorem \ref{theorem:identuptoscale} and Assumption \textbf{A.3}, $f(\mathbf{X_1}) = c^{-1}*f^{'}(\mathbf{X_1}) = 1$ and $g(\mathbf{X_2}) = c*g^{'}(\mathbf{X_2}) = 1$. Since equation \ref{eq:prob-constraint} applies to $f^{'}(\cdot)$ and $g^{'}(\cdot)$ also, the only possible value for $c$ is $1$.
\end{proof}

The above proof can be easily extended to three or more unobserved functions where the observed composite event is a product of all of them. Note that a similar proof works in the setting where the observed event is a known weighted sum of two unobserved events where instead of the scale factor, the estimated functions are misaligned by positive and negative offset values respectively, of equal magnitude. The equality corresponding to Equation~\ref{eq:equality} is $a(f(\mathbf{X_1}) - f^{'}(\mathbf{X_1})) = b(g^{'}(\mathbf{X_2}) - g(\mathbf{X_2}))$.
% \begin{equation}
%     a(f(\mathbf{X_1}) - f^{'}(\mathbf{X_1}))= b(g^{'}(\mathbf{X_2}) - g(\mathbf{X_2}))
% \end{equation}
We leverage this proof in some of the involved scenarios we encounter in the experiments on real data later.

\subsubsection{Validity of Assumptions}
\label{subsubsec:validity-of-assumptions}
The above proof shows that the unobserved events are identified if the covariates determining them have independent components (Assumption \textbf{A.1}). We experiment on synthetic data to test sensitivity of our model to this assumption. The results produced by our model even when \textbf{A.1} is violated are promising in comparison to the setting where \textbf{A.1} holds. Moreover, given the large number of features compared to the small number of events to be estimated in any real-world dataset, it is highly unlikely that the unobserved events depend only on shared features and have no independent covariates. In fact, real data used in this paper for experimentation does exhibit independent variation for the unobserved estimated functions, known from domain knowledge.
%SS comment: we can add how targeting might be based on likelihood of conversion, whereas click action is based on interestingness of the message.

We test our model in scenarios where \textbf{A.2} is violated in synthetic data. The good performance in this setting indicates that, in practice, the model is able to identify the covariates that determine the respective events. We also evaluate our model in a setting where both \textbf{A.1} and \textbf{A.2} are violated together.

%\subsubsection{Identification of Scale Factor without Assumption \textbf{A.3}}
The third assumption (Assumption \textbf{A.3}), as stated earlier, might not be satisfied in real data. In absence of \textbf{A.3}, the two functions are identified only up to a scale. Therefore, we present an alternate strategy to identify the scale even when \textbf{A.3} is invalid. We do not observe the events estimated by $f^{'}(.)$ and $g^{'}(.)$, however, it is possible to get data at the aggregate level. For example, it is difficult to know which customer was shown a search ad but it might be easy to get data on the fraction of customers who were shown a search ad. We use this information for identifying the scale parameter of the unobserved events. This approach is discussed next.
\vspace{-0.5em}

\subsection{Estimation}
\label{subsec:estimation}
The proposed solution strategy follows from the above proof. We use a simple multi-layer perceptron (MLP) neural network architecture to estimate the unobserved events of interest that constitute the observed composite events as probability functions of the input features. The output layer corresponding to the estimated values of the unobserved events therefore has a softmax activation function, constraining the output between 0 and 1. The observed composite event is estimated from the output of the network using the known functions. In Example 1, the estimated value of $P_{a,c}$ is obtained as a simple product of the estimations for $P_c$ and $P_a$ by the network. The network parameters are updated based on a binary cross-entropy (BCE) loss function applied to the estimated and observed values of the composite event.

As explained in Subsection \ref{subsubsec:validity-of-assumptions}, a simple loss function based on observed composite events can help identify the unobserved events only up to a scale in a realistic setting and the scale needs to be identified using the aggregate level data. Therefore, we define a novel custom loss term based on the difference between the sample average of the estimated probabilities and the actual aggregate fractions. This aggregate loss term, when added to the BCE loss function, helps in learning the unobserved events at their correct scales. Further, we also perform exponential smoothing of this added term across batches of data to guard against any drastic variation across these batches. Therefore the different training strategies, in terms of the loss function, that are employed here are:
\begin{itemize} 
    \item[\textbullet] \texttt{BCEL} - using a simple binary cross entropy (BCE) loss function for each of the observed variables and summing up,
    \begin{equation}
    \label{eq:BCE_loss}
        %\mathcal{L}_b = \sum_{V \in \mathcal{V}} \frac{1}{|\eta_b|} \sum_{i \in \eta_b} -\{v_i * log(\hat{v}_i) + (1-v_i)*log(1-\hat{v}_i)\}
        \mathcal{L}_b = \sum_{Y \in \mathcal{Y}} \frac{1}{|\eta_b|} \sum_{i \in \eta_b} -\{y^i * log(\hat{P}_Y^i) + (1-y^i)*log(1-\hat{P}_Y^i)\}
    \end{equation} 
    where $\mathcal{L}_b$ is the BCE loss of data batch $b$, $\eta_b$ is the set of data samples in $b$, $\mathcal{Y}$ is the set of all observed variables, $y^i$ is the true value of the $i^{th}$ sample of variable $Y$ and $\hat{P}_Y^i$ is the estimate of $P_Y$, the probability $\prob(Y=1)$. Note that $\hat{P}_Y^i$ may be obtained by computing the product of the estimated probabilities of other observed and unobserved variables. 
    
    \item[\textbullet] \texttt{AGGL} - As noted above, the scale of the unobserved events is not identified without Assumption \textbf{A.3}. An additional aggregate loss term is added to the loss function to enable the model to learn the event probabilities at the correct scale factor. This loss term is based on the difference in the estimated and actual probabilities at the aggregate level. It is calculated as $\Delta_b = \left (P_{Y1} - \hat{P}_{Y1}^b \right )^2 + \left (P_{Y2} - \hat{P}_{Y2}^b \right )^2$,
%     \begin{equation}
%     \label{eq:agg_loss}    
% %        \Delta_b = \sum_{Z \in \mathcal{Z}}\left (P_Z - (\widehat{P_Z})_b \right )^2
%         \Delta_b = \left (P_{Y1} - \hat{P}_{Y1}^b \right )^2 + \left (P_{Y2} - \hat{P}_{Y2}^b \right )^2
%     \end{equation}
    where $\Delta_b$ is the aggregate loss of batch $b$, $P_{Yj}$ is the actual aggregate probability of the event $(Y_j=1)$, i.e. fraction of all samples where variable $Y_j=1$, $\hat{P}_{Yj}^b = |\eta_b|^{-1}\sum_{i\in \eta_b} \hat{P}_{Yj}^i$ is the estimated value of the same aggregate probability over batch $b$. 
    
    The overall loss function in this case is $\mathcal{L}_b + \lambda \Delta_b$ for some constant weight $\lambda$. Minimization of the loss function with aggregate loss, $\Delta_b$ over each training batch allows identification of correct scale of the probabilities of the unobserved events.

    \item[\textbullet] \texttt{SAGG} - The loss function in this case takes exponential smoothing of the aggregate loss term $\Delta_b$ over different mini-batches of data. Since we estimate low probability events in our experiments, the aggregate loss may vary drastically across batches and hence needs to be smoothed by modifying the estimated aggregate probability to $\hat{P}_{Yj}^{sb}$.
    \begin{equation}
    \label{eq:smoothing}
        \hat{P}_{Yj}^{sb}=\begin{cases}
        \hat{P}_{Yj}^{b} & \text{ if } b=0 \\ 
        \alpha*\hat{P}_{Yj}^{b} + (1-\alpha) * \hat{P}_{Yj}^{sb-1}& \text{ otherwise }
        \end{cases}
    \end{equation}
    where $\alpha$ is the smoothing weight and $\hat{P}_{Yj}^{sb}$ is the smoothed version of the estimated aggregate probability. This is similar to Batch Renormalization proposed in \cite{Ioffe2017BatchRenormalization}.
\end{itemize}

Note that even though we prove identifiability of the unobserved events under certain settings given by the assumptions in subsection \ref{subsec:identifiability}, we test and show that our solution methodology performs well in settings where the assumptions are violated.

The next section describes the datasets used to conduct the experiments.

% However, even if there is no aggregate data available, the independence condition can be exploited to identify the scale parameter. 
% \begin{theorem} \label{theorem:momentcondition}
% Given Equation~\ref{eq:prodprob} and assumption \textbf{A.3}, we have $\expect_{X_1,X_2}(h(\cdot)) = \expect_{X_1}(f(\cdot))*\expect_{X_2}(g(\cdot))$.
% \end{theorem}
% \begin{proof}
% The independence assumption (\textbf{A.3}) implies:
% $$\expect_{X_1,X_2}(f(\cdot)*g(\cdot)) = \expect_{X_1}(f(\cdot))*\expect_{X_2}(g(\cdot)).$$
% \end{proof}
% The above result pins down the expected value of the observed event as a product of the expected values of the unobserved events. The second approach to identify the scale parameter is by minimizing the difference between the sample mean of the observed event and the estimated means of the unobserved events. \ayush{this equation needs to be modified}
% \begin{equation} \label{eq:momentmin}
% \min_c \left(\frac{1}{n}\sum_i Y_i - \frac{1}{n}\sum_i \mathbf{I}(f_{i}\geq 0.5c)*\frac{1}{n}\sum_i \mathbf{I}(cg_{i} \geq 0.5 )\right)^2
% \end{equation}
% Where $Y_i$ is the event corresponding to $P_{a,c}$ in data for $i^{th}$ customer.
% This can be done in following two ways:
% \begin{enumerate}
%     \item As a post processing step after the model is estimated.
%     \item By including Equation~\ref{eq:momentmin} directly in the loss function. We follow the literature on using a moment conditions in a neural network loss function \cite{ravuri2018PMLR}.
% \end{enumerate}
\vspace{-0.5em}
\section{Data}
The approach presented in Section~\ref{sec:model} is tested on a real customer behavior dataset. This dataset allows us to partially test the approach by artificially suppressing the observed data. In addition, for testing the assumptions and refining the approach, we use synthetic datasets generated under three different settings.
\vspace{-0.5em}
\subsection{Customer Behavior Data}
\label{subsection:realdata}
The real data is based on recorded interactions of customers with a brand. The data is from a product company that has a website and focuses marketing efforts on digital channels. There are multiple sources of data that record interactions and transaction between the firm and its customers. 
% re-word this onward
These interaction events are recorded in four different data sources - web-analytics data, display ad impression data, email interaction data, and product usage data. We merge these data sources. 
Web-analytics data is stored in the form of a clickstream that records the online activities of a customer on the company's website. 
Roughly, there is one row of data for each URL visited on the company's website. 
These visits include pages with information on the product features, product help, download trial versions or checkout.
Each row contains information about the customer's device, geography, source of visit, URL, time-stamp, product purchased, etc.
% up to here
The visits from the search channel are recorded in the web-analytics dataset as well.
When a customer performs keyword search on a browser, the firm may decide to algorithmically bid on the search keyword.
A link to the firms online properties may be shown to the customer through a search ad or an organic search result. 
Once the customer clicks on the link, the data is recorded in the firm's clickstream. 
Hence, we do not observe the search event and it is not possible to track the ad shown event when the customer visits the website through an organic search link.
For the Search use-case, we use a dataset of 154982 customers. 
Of these, 1486 customers clicked on a paid search ad and 7733 clicked on an organic search result.
The web-analytics dataset is used to create features to predict email and search related events of interest.

% re-word this onwards
% The striked out text can be excluded
%\st{The display ad data provides information on ad impression on company's own website as well as other websites. 
%This dataset stores which ads were shown to the customer and whether the customer clicked on the ad and many other ad-related properties.
% up to here
%The display dataset contains information on display impression as well as click.}
% re-word this onwards
The email interaction dataset consists of information related to the emails sent by the organization to its customers. 
This dataset includes information such as whether a customer opened the email, clicked on a link in the email, unsubscribed to emails from the company, description of the email, etc.
The experiments on the email data are carried out on a set of 174059 customers.
Out of these customers, 38539 were sent an email, 28041 had opened the email and 1873 had clicked on the email. 

Finally, product usage data contains information on a customer's interactions with company's applications. 
Each row of the data stores information on the events such as application launch, application download etc.
Each dataset described above uses the same identifier for the customer, this identifier is used for merging the datasets.

In summary, the list of observed and unobserved events is as follows: 1) Search channel: Unobserved - Customer searches or not, Ad shown or not; Observed - Ad Click, Organic Search Click; 2) Email channel: Observed - Email Sent, Email opened, Email clicked. Note that, in experiments, email-sent event is hidden from the algorithm and used only for validation of the approach. More on this in Section~\ref{sec:experiments}.

% \begin{enumerate}
%     \item Search channel: Unobserved - Customer searches or not, Ad shown or not; Observed - Ad Click, Organic Search Click
% %    \item \st{Display channel: Unobserved - Ad seen; Observed - Ad impression, Ad click}
%     \item Email channel: Observed - Email Sent, Email opened, Email clicked. Note that, in experiments, email-sent event is hidden from the algorithm and used only for validation of the approach. More on this in Section~\ref{sec:experiments}.
% \end{enumerate}

For feature creation, we adopt a two period approach - that is, for each customer, features are constructed from his interaction with the brand for a fixed period - and evaluations are done for observations post this period. The data comprises  various events - when a customer downloaded an app, when he was shown an ad, when he clicked on a paid search etc. For each customer we extracted 144 features from this data - each of  these features is a measure of customer interaction for a particular event. Every experiment that we perform uses these 144 features to predict outcomes. For example, using the information of when customers were sent email, we constructed features for recency of email sent, frequency - the number of times  he was sent email over the period of analysis. 

In summary each row of the data contains 144 customer features, Search Ad Click, Organic Search Click, Email Sent, Email Opened and Email Click during the post feature creation period.
% up to here
\vspace{-0.5em}
\subsection{Synthetic Data}
\label{subsection:syntheticdata}
The customer behavior dataset is a good source to partially validate the usefulness of the approach on a real-world data. However, in absence of ground truth and uncertainty about whether the assumptions are valid, it is important to test the approach under laboratory conditions. For this purpose, the proposed approach is tested on synthetically generated datasets. In synthetically generated data, robustness of the approach is tested under different settings of Assumptions \textbf{A.1} and \textbf{A.2} and when full support assumption (\textbf{A.3}) is violated.

The synthetic data consists of four scalar input features denoted by $\mathbf{X} = (x_1, x_2, x_3, x_4)^T$, two unobserved binary variables $Y_1$, $Y_2$, and one observed binary output variable $Y$. Each of the binary variables also has an associated probability which, in fact, determines its value - $0$ or $1$. Therefore, the probabilities of the events $(Y_1 = 1)$ and $(Y_2 = 1)$ are determined first, as sigmoids of linear functions of the input features. The set of features that determine $Y_1$ and $Y_2$ are $\mathbf{X_1}$ and $\mathbf{X_2}$ respectively where $\mathbf{X_1}, \mathbf{X_2} \subseteq \mathbf{X}$. The probability of success for the observed outcome $Y$ is the product of the probabilities of success for $Y_1$ and $Y_2$. The four features are sampled from zero mean Gaussian distributions with the standard deviations varying between $1$ and $5$. The total number of samples in the dataset are $100000$ which are randomly divided into training, validation and test sets of sizes $55000$, $20000$ and $25000$ respectively. Finally, $Y$, $Y_1$ and $Y_2$ are generated by performing Bernoulli trials with the respective probabilities of success.

There are different scenarios that we test for in our experiments and each involves a different data generation process. They are as follows:

\subsubsection{Independent Covariates and Known \texttt{(IND COV KWN)}}
This is the case when the features determining $Y_1$ and $Y_2$ are mutually exclusive and therefore independent. Thus, $\mathbf{X_1}=\{x_1, x_2\}$ and $\mathbf{X_2}=\{x_3, x_4\}$. The identity of the features that determine $Y_1$ and $Y_2$ is also known to the modeler. This is the setting where both \textbf{A.1} and \textbf{A.2} are satisfied.

\subsubsection{Independent Covariates and Unknown \texttt{(IND COV UNK)}}
In this case the data generating process is same as above whereby the two unobserved variables $Y_1$ and $Y_2$ are functions of mutually exclusive set of features. The only difference is that during modeling it is not known which input features determine which variable. This is the usual scenario a modeler is likely to face where \textbf{A.1} holds but not \textbf{A.2}.

\subsubsection{Partial Overlap and Unknown \texttt{(PAR OV UNK)}}
Here, we consider a case in which the two unobserved variables share some covariates but still have some independent components. The goal is to test the approach under the realistic scenario where some features are shared. Here, $\mathbf{X_1}=\{x_1, x_2, x_3\}$ and $\mathbf{X_2}=\{x_2, x_3, x_4\}$.
% \begin{equation*}
%     \centering
%     \begin{split}
%         P_{Y_1} = \sigma(w_{10} +w_{11}x_1 + w_{12}x_2 + w_{13}x_3)\ &;\ Y_1 = \mathcal{B}(P_{Y_1})\\
%         P_{Y_2} = \sigma(w_{20} +w_{22}x_2 + w_{23}x_3 + w_{24}x_4)\ &;\ Y_2 = \mathcal{B}(P_{Y_2})\\
%         P_Y = P_{Y_1} * P_{Y_2}\ &;\ Y = \mathcal{B}(P_{Y})
%     \end{split}
% \end{equation*}
In this case, the modeler does not know which features determine which event. Therefore, this case also satisfies \textbf{A.1} but not \textbf{A.2}.

\subsubsection{Complete Overlap \texttt{(COM OV)}}
This is a case where the two unobserved variables share all the features with a perfect overlap and without any independent component. Although this case is unlikely to occur in a real-world setting, we use this to test our model when there is no independent variation in data. Since all the 4 covariates determine both the variables, $\mathbf{X_1}=\{x_1, x_2, x_3, x_4\}$ and $\mathbf{X_2}=\{x_1, x_2, x_3, x_4\}$. This is the only setting where Assumption \textbf{A.1} is violated.
\\
\\
In all the above cases, the probabilities and the binary variables are given by $P_{Y_1} = \sigma(\mathbf{w_1^T X_1}), Y_1 = \mathcal{B}(P_{Y_1})$; $P_{Y_2} = \sigma(\mathbf{w_2^T X_2}), Y_2 = \mathcal{B}(P_{Y_2})$, and $P_Y = P_{Y_1} * P_{Y_2}, Y = \mathcal{B}(P_{Y})$.
Here $\mathbf{w_1}$ and $\mathbf{w_2}$ are weight vectors of same cardinality as $\mathbf{X_1}$ and $\mathbf{X_2}$ respectively. Further, $P_{Y_1}, P_{Y_2}, P_Y$ are the probabilities $\prob(Y_1=1), \prob(Y_2=1), \prob(Y=1)$ respectively and $\mathcal{B}(p)$ is Bernoulli trial with probability of success equal to $p$. 
\vspace{-0.5em}
\section{Experiments} \label{sec:experiments}
We perform various experiments to test our solution approach proposed in Section \ref{sec:model}. In all the experiments, we compare the three training strategies described in Subsection \ref{subsec:estimation}:
\begin{itemize} 
    \item[\textbullet] \texttt{BCEL} - using a simple binary cross entropy (BCE) loss
    \item[\textbullet] \texttt{AGGL} - including an additional aggregate loss term to estimate the scale
    \item[\textbullet] \texttt{SAGG} - performing exponential smoothing of the aggregate loss term to prevent drastic variations across mini-batches
\end{itemize}
The estimation results are evaluated by comparing the estimated and actual probability of each variable in test data. The correctness of the estimated value is defined in terms of two error metrics - Mean Squared Error (MSE = $ \frac{1}{|\eta_t|}\sum_{i \in \eta_t}\left ( P_{Yj}^i - \hat{P}_{Yj}^i \right )^2$) and Mean Absolute Percentage Error (MAPE= $ \frac{1}{|\eta_t|}\sum_{i \in \eta_t}\frac{|P_{Yj}^i - \hat{P}_{Yj}^i|}{P_{Yj}^i} $). Where $\eta_t$ is the set of all test samples, $P_{Yj}^i$ is the actual probability of $Y_j$ for the $i^{th}$ sample and $\hat{P}_{Yj}^i$ is the corresponding estimated probability.
% \begin{equation*}
% \begin{split}
%     MSE &= \frac{1}{|\eta_t|}\sum_{i \in \eta_t}\left ( P_{Yj}^i - \hat{P}_{Yj}^i \right )^2\\
%     MAPE &= \frac{1}{|\eta_t|}\sum_{i \in \eta_t}\frac{|P_{Yj}^i - \hat{P}_{Yj}^i|}{P_{Yj}^i}
% \end{split}
% \end{equation*}
% where $\eta_t$ is the set of all test samples, $P_{Yj}^i$ is the actual probability of $Y_j$ for the $i^{th}$ sample and $\hat{P}_{Yj}^i$ is the corresponding estimated probability.
\vspace{-0.5em}
\subsection{Experiments on Synthetic Data}
\setlength{\tabcolsep}{4pt}
\begin{table*}[!ht]
\centering
\caption{Comparison of MSE\textsuperscript{*} and MAPE\textsuperscript{$\dagger$} for $Y$, $Y_1$, $Y_2$ using the three training strategies in the 4 synthetic datasets in absence of Assumption \textbf{A.3}.}
\vspace{-0.5em}
\label{tab:synthetic_low_prob}
\resizebox{0.98\textwidth}{!}{%
\begin{tabular}{cc|ccc|ccc|ccc|ccc}
\toprule
\multirow{2}{*}{\begin{tabular}[c]{@{}c@{}}\\ \ $\bm{Var}$\  \end{tabular}} & 
\multirow{2}{*}{\begin{tabular}[c]{@{}c@{}}\\ \ $\bm{Metric}$\ \ \end{tabular}} &
\multicolumn{3}{c|}{\texttt{IND COV KWN}} & \multicolumn{3}{c|}{\texttt{IND COV UNK}} & \multicolumn{3}{c|}{\texttt{PAR OV UNK}} & \multicolumn{3}{c}{\texttt{COM OV}}\\
  &  & \texttt{BCEL} & \texttt{AGGL} & \texttt{SAGG} & \texttt{BCEL} & \texttt{AGGL} & \texttt{SAGG}& \texttt{BCEL} & \texttt{AGGL} & \texttt{SAGG} & \texttt{BCEL} & \texttt{AGGL} & \texttt{SAGG} \\ 
\midrule
\ $Y$\               & \ \ \ MSE\textsuperscript{*}\ \ \    & \  0.04\   & \  0.04\   & \  0.04\ \  & \  0.14\   & \  0.14\   & \  0.14\ \ & \  0.33\   & \  \textbf{0.32}\   & \  \textbf{0.32}\ & \  \textbf{0.40}\   & \  \textbf{0.40}\   & \  0.49\   \\ 
\ (Obs)\            & \ \ \ MAPE\textsuperscript{$\dagger$}\ \ \   & \  7.84\   & \  \textbf{5.39}\   & \  5.41\  & \  13.29\   & \  \textbf{12.83}\   & \  13.05\  & \  9.68\   & \  \textbf{9.59}\   & \  \textbf{9.59}\ & \  17.76\   & \  \textbf{16.17}\   & \  20.13\  \\ 
\midrule
\ $Y_1$\              & \ \ \ MSE\textsuperscript{*}\ \ \    & \  77.55\   & \  \textbf{0.68}\   & \  0.69\ \ & \  58.07\   & \  \textbf{15.78}\   & \  16.23\ \ & \  48.77\   & \  7.67\   & \  \textbf{7.66}\ & \  51.17\   & \  11.70\   & \  \textbf{11.66}\  \\ 
\ (Unobs)\          & \ \ \ MAPE\textsuperscript{$\dagger$}\ \ \   & \  51.03\   & \  \textbf{4.54}\   & \  4.56\ \ & \  41.98\   & \  \textbf{22.89}\   & \  23.20\ \ & \  35.42\   & \  16.87\   & \  \textbf{16.85}\ & \  36.81\   & \  19.91\   & \  \textbf{19.72}\  \\
\midrule
\ $Y_2$\              & \ \ \ MSE\textsuperscript{*}\ \ \    & \  20.43\   & \  \textbf{0.10}\   & \  \textbf{0.10}\ \ & \  9.19\   & \  \textbf{1.07}\   & \  1.11\ \ & \  2.69\   & \  \textbf{2.10}\   & \  \textbf{2.10}\ & \  4.41\   & \  \textbf{4.11}\   & \  4.27\  \\
\ (Unobs)\          & \ \ \ MAPE\textsuperscript{$\dagger$}\ \ \   & \  109.58\   & \  \textbf{2.69}\   & \  2.70\ \ & \  86.28\   & \  \textbf{22.12}\   & \  22.58\ \ & \  73.05\   & \  15.34\   & \  \textbf{15.31}\  & \  69.16\   & \  \textbf{21.60}\   & \  24.95\ \\
\bottomrule
\end{tabular}%
}
\end{table*}
We validate our proposed solution on the synthetic data generated for the 4 cases as described in Subsection \ref{subsection:syntheticdata}. The model in each of the cases has 2 nodes in the output layer to predict the probability of $Y_1$ and $Y_2$ respectively. The loss is defined on the product of the 2 outputs which is observed as $Y$ in the data. The training is run for $150$ epochs on batches of $128$ data samples and the trained model at epoch with minimum validation loss is selected for testing.
Although the models are trained using the observed binary variable $Y$, the performance of these models is evaluated by validating against the underlying true probabilities of all the three variables, i.e. $P_Y, P_{Y_1}, P_{Y_2}$.
 
As discussed in Subsection \ref{subsection:syntheticdata}, we evaluate our approach on synthetic data in four different scenarios:
\begin{itemize}
    \item[\textbullet] \texttt{IND COV KWN} - Since the two set of covariates are known to independently determine the unobserved variables $Y_1$ and $Y_2$, we train two separate fully connected feed forward neural networks to learn $Y_1$ from $\{x_1,x_2\}$ and $Y_2$ from $\{x_3,x_4\}$ respectively. Both the networks have a single hidden layer of 3 nodes with ReLU activation function and 1 output node with Sigmoid activation. 
    
    \item[\textbullet] \texttt{IND COV UNK} - Since the independence of relationship on the covariates, though existing, is assumed to be unknown to the modeler, the network architecture is a fully connected feed forward one with all the features connected to both the unobserved variables. The network has a single hidden layer of size 3 and ReLU activation function while the output layer has 2 nodes with Sigmoid activation.

    \item[\textbullet] \texttt{PAR OV UNK} -  The two unobserved variables in this case cannot be determined independently. To a modeler, this case is same as \texttt{IND COV UNK} and hence the network architecture is also same as in the previous case.
    
    \item[\textbullet] \texttt{COM OV} - Since all the covariates determine both the unobserved variables, the network architecture in this case is a fully connected feed forward one just as in the previous two cases. 
\end{itemize}
\makeatletter{\renewcommand*{\@makefnmark}{*}
\footnotetext{MSE values are reported as parts per thousand, i.e. multiplied by $10^3$.}\makeatother}
\makeatletter{\renewcommand*{\@makefnmark}{$\dagger$}
\footnotetext{MAPE values are reported as percentages, i.e. multiplied by $10^2$.}\makeatother}

In the first set of synthetic data experiments, we compare the three training strategies in all of the above cases under the more practical setting of absence of Assumption \textbf{A.3}. Since we deal with low probability events in the real world scenario, we allow $P_{Y_1}$ and $P_{Y_2}$ to go up to a maximum of only $0.6$. In addition, we perform another experiment to test Corollary \ref{corollary:scaleident} on synthetic data. We do so by testing the performance of our model, essentially if the scale is identified correctly, when $P_{Y_1}$, $P_{Y_2}$ and hence $P_Y$ are all equal to $1$ for at least some instances. 
\setlength{\tabcolsep}{4pt}
\begin{table*}[!ht]
\centering
\caption{MSE\textsuperscript{*} and MAPE\textsuperscript{$\dagger$} for $Y$, $Y_1$, $Y_2$ using the three training strategies on the 4 synthetic datasets under Assumption \textbf{A.3}.}
\vspace{-0.5em}
\label{tab:synthetic_high_prob}
\resizebox{0.98\textwidth}{!}{%
\begin{tabular}{cc|ccc|ccc|ccc|ccc}
\toprule
\multirow{2}{*}{\begin{tabular}[c]{@{}c@{}}\\ \ $\bm{Var}$\  \end{tabular}} & 
\multirow{2}{*}{\begin{tabular}[c]{@{}c@{}}\\ \ $\bm{Metric}$\ \ \end{tabular}} &
\multicolumn{3}{c|}{\texttt{IND COV KWN}} & \multicolumn{3}{c|}{\texttt{IND COV UNK}} & \multicolumn{3}{c|}{\texttt{PAR OV UNK}} & \multicolumn{3}{c}{\texttt{COM OV}} \\
  &  & \texttt{BCEL} & \texttt{AGGL} & \texttt{SAGG} & \texttt{BCEL} & \texttt{AGGL} & \texttt{SAGG}& \texttt{BCEL} & \texttt{AGGL} & \texttt{SAGG} & \texttt{BCEL} & \texttt{AGGL} & \texttt{SAGG} \\ 
\midrule
\ $Y$\               & \ \ \ MSE\textsuperscript{*}\ \ \    & \  \textbf{0.11}\   & \  0.12\   & \  0.12\ \  & \  2.54\   & \  \textbf{2.49}\   & \  \textbf{2.49}\ \ & \  \textbf{2.65}\   & \  2.67\   & \  2.67\  & \  0.57\   & \  \textbf{0.55}\   & \  \textbf{0.55}\ \\ 
\ (Obs)\            & \ \ \ MAPE\textsuperscript{$\dagger$}\ \ \   & \  2.44\   & \  \textbf{2.42}\   & \  2.34\  & \  13.95\   & \  \textbf{13.49}\   & \  \textbf{13.49}\  & \  \textbf{24.00}\   & \  25.28\   & \  25.20\ & \  \textbf{9.18}\   & \  10.13\   & \  10.11\  \\ 
\midrule
\ $Y_1$\              & \ \ \ MSE\textsuperscript{*}\ \ \    & \  \textbf{0.23}\   & \  0.24\   & \  0.24\ \ & \  62.09\   & \  \textbf{10.69}\   & \  \textbf{10.69}\ \ & \  43.56\   & \  \textbf{14.73}\   & \  \textbf{14.73}\ & \  172.33\   & \  \textbf{35.06}\   & \  35.17\  \\ 
\ (Unobs)\          & \ \ \ MAPE\textsuperscript{$\dagger$}\ \ \   & \  1.34\   & \  1.34\   & \  1.34\ \ & \  22.21\   & \  \textbf{10.61}\   & \ \textbf{10.61}\ \ & \  18.88\   & \  \textbf{12.82}\   & \  \textbf{12.82}\ & \  40.24\   & \  \textbf{21.70}\   & \  21.74\  \\
\midrule
\ $Y_2$\              & \ \ \ MSE\textsuperscript{*}\ \ \    & \  0.04\   & \  \textbf{0.03}\   & \  \textbf{0.03}\ \ & \  8.72\   & \  \textbf{5.54}\   & \  \textbf{5.54}\ \ & \  8.06\   & \  7.21\   & \  \textbf{7.20}\  & \  10.17\   & \  \textbf{2.65}\   & \  2.66\ \\
\ (Unobs)\          & \ \ \ MAPE\textsuperscript{$\dagger$}\ \ \   & \  1.67\   & \  1.66\   & \  \textbf{1.56}\ \ & \  35.56\   & \  \textbf{13.07}\   & \ 13.08\ \ & \  33.47\   & \  25.11\   & \  \textbf{25.10}\  & \  125.60\   & \  \textbf{17.09}\   & \  17.19\ \\
\bottomrule
\end{tabular}%
}
\end{table*}
%Further, to ensure that Assumption \textbf{A.3} is not violated we perform this experiment for only \texttt{IND COV} case. 

\vspace{-0.5em}
\subsection{Experiments on Customer Behavior Data} \label{subsec:exprealdata}
We test our approach on the real world data of brand-customer interactions described in Subsection \ref{subsection:realdata}.
% For Aditya - I have removed the description of data from the Experiment section. Moved that to the Data section
% For Aditya - List all experiments with real data. Also add why each experiment is being conducted. It should be in the same order that we present the result. 1) Consistency experiment - What is it? And what does it show that is why we are doing it? 2) Hide an observed event. 
We perform our experiments on this dataset for the following two use cases:
\begin{itemize}
\item[\textbullet] Email - In this case, the variables that are observed in the data are the actions \{\texttt{Email Send}, \texttt{Email Open}, \texttt{Email Click}\}.  
\item[\textbullet] Search - In the search case, only two events are observed in the data \texttt{Paid Search Ad Click} and \texttt{Organic Search Click}. Hereafter, they are referred as \texttt{Ad Click} and \texttt{Organic Click} respectively.  
\end{itemize}
The experiments on the Email data, where we observe all events, is to validate the approach. This is done by suppressing observed events from the algorithm and using these as ground truth for validation. The experiments on Search data, where events are not observed, are used for weaker validation called Consistency experiments, described next, and showing how to extend the approach to more complicated events. The set of experiments performed on these datasets can be broadly categorized as Correctness and Consistency Experiments.
\subsubsection{Correctness Experiment}
Since the events of interest are not observed in the Search data, the correctness experiment is possible only on Email data. We perform this experiment in the realistic setting where an analyst has access to the data on \texttt{Email Open} and \texttt{Email Click} as these are customer actions easily recorded by analytics tools, whereas the event \texttt{Email Send} is unobserved to him as it is performed by the marketer. To simulate this setting, \texttt{Email Send} is artificially hidden from the model during training. The model is therefore trained to estimate the probabilities of the 3 unobserved events \{\texttt{Email Send}, \texttt{Email Open|Send}, \texttt{Email Click|Open}\} while the BCE loss $\mathcal{L}_b$ is defined on the 2 observed events \{\texttt{Email Open}, \texttt{Email Click}\}. The estimated probability of observed events is obtained from the network outputs as $\prob[\texttt{Open}]=\prob[\texttt{Send}]*\prob[\texttt{Open}|\texttt{Send}]$ and $\prob[\texttt{Click}]=\prob[\texttt{Open}]*\prob[\texttt{Click}|\texttt{Open}]$.

% \begin{equation}
% \label{eq:email}
%     \begin{split}
%         \prob[\texttt{Open}]&=\prob[\texttt{Send}]*\prob[\texttt{Open}|\texttt{Send}]\\
%         \prob[\texttt{Click}]&=\prob[\texttt{Open}]*\prob[\texttt{Click}|\texttt{Open}]
%     \end{split}
% \end{equation}
The architecture of the network is such that it has 3 separate fully connected units each consisting of 4 hidden layers of size 70, 40, 20 and 10 respectively and a single output node. Each unit learns to predict one of the 3 unobserved events. The probabilities of the two events observed by the algorithm are estimated as in above equations.
%~\ref{eq:email}.
\makeatletter{\renewcommand*{\@makefnmark}{*}
\footnotetext{MSE values are reported as parts per thousand, i.e. multiplied by $10^3$.}\makeatother}
\makeatletter{\renewcommand*{\@makefnmark}{$\dagger$}
\footnotetext{MAPE values are reported as percentages, i.e. multiplied by $10^2$.}\makeatother}

The dataset is broken down into training, validation and test sets of 95297, 35247 and 43515 customers respectively. The training is run for a maximum of 50 epochs on batches of 1024 instances, while the validation loss is used to select and return the best model. The performance of the model is evaluated in terms of MSE and MAPE between the predicted and ground truth probabilities of the 5 variables. In order to generate the ground truth probabilities, we train a XGBoost Classifier on the 3 events observed in data and probabilities of the other 2 events are obtained using equations above. We compare our estimated probabilities to this ground truth because this is what one would do for prediction, had \texttt{Email Send} been observed by the modeler. Note that in the \texttt{AGGL} training scenario, the aggregate loss term $\Delta_b$ includes the sample average of all the 5 variables, obtained easily using the ground truth classifier model.
\setlength{\tabcolsep}{3pt}
\begin{table*}[!ht]
\centering
\caption{Comparison of MSE\textsuperscript{$\ddag$} and MAPE for all the variables in the Email data using the three training strategies.}
\vspace{-0.5em}
\label{tab:real_data_correct}
\resizebox{0.98\textwidth}{!}{
\begin{tabular}{c|ccc|ccc|ccc|ccc|ccc}
\toprule
$\ \ \bm{Variable}$\ \ &
\multicolumn{3}{c|}{\texttt{Send} (Unobs)} & \multicolumn{3}{c|}{\texttt{Open|Send} (Unobs)} & \multicolumn{3}{c|}{\texttt{Open} (Obs)} &
\multicolumn{3}{c|}{\texttt{Click|Open} (Unobs)} & \multicolumn{3}{c}{\texttt{Click} (Obs)} \\
\cmidrule{1-1}
$\ \bm{Metric}$\ & \texttt{BCEL} & \texttt{AGGL} & \texttt{SAGG} & \texttt{BCEL} & \texttt{AGGL} & \texttt{SAGG}& \texttt{BCEL} & \texttt{AGGL} & \texttt{SAGG} 
& \texttt{BCEL} & \texttt{AGGL} & \texttt{SAGG} & \texttt{BCEL} & \texttt{AGGL} & \texttt{SAGG} \\ 
\midrule
\ \ \ MSE\textsuperscript{$\ddag$}\ \ \    & \  10.83\   & \  3.39\   & \  \textbf{3.27}\ \  & \  385.9\   & \  333.2\   & \ \textbf{332.8}\ \ & \  \textbf{0.83}\   & \   \textbf{0.83}\   & \  0.89\ & \  \textbf{0.08}\   & \  \textbf{0.08}\   & \  0.09\ \  & \  0.07\   & \  0.07\   & \  0.07\  \\ 
\ \ \ MAPE\ \ \   & \  8.88\   & \  \textbf{0.96}\   & \ \textbf{ 0.96}\ \ & \  \textbf{0.68}\   & \  0.78\   & \  0.78\ \ & \  0.72\   & \  \textbf{0.66}\   & \  0.68\ & \  2.70\   & \  \textbf{0.80}\   & \  1.02\ \ & \  1.85\   & \  \textbf{0.43}\   & \  0.67\  \\ 
\bottomrule
\end{tabular}
}
\end{table*}

\subsubsection{Consistency Experiment}
Since the Search case does not have enough observed events to check the correctness of the estimated probabilities, we check the consistency of the results across different random initializations of the neural net parameters. If the unobserved events are not identified correctly and only the product (observed events) is identified, then the estimate of the unobserved events across runs should generate different results.

The model in this case is trained to directly estimate the probabilities of 5 unobserved events \{\texttt{Search}, \texttt{Ad Shown|Search}, \texttt{Ad Click|Ad Shown}, \texttt{Organic Click|Ad Shown}, \texttt{Organic Click|Ad Not Shown} \} from the probability of the 2 observed events \{\texttt{Ad Click}, \texttt{Organic Click}\} as follows:
\begin{equation}
\label{eq:search}
    \begin{split}
        \prob[\texttt{Ad Click}] &= \prob[\texttt{Ad Shown}]*\prob[\texttt{Ad Click|Ad Shown}]\\
        \prob[\texttt{Ad Shown}] &= \prob[\texttt{Search}]*\prob[\texttt{Ad Shown|Search}]\\
        \prob[\texttt{Ad Not Shown}] &= \prob[\texttt{Search}]*(1-\prob[\texttt{Ad Shown|Search}])\\
        \prob[\texttt{Organic Click}] &= \prob[\texttt{Ad Shown}] * \prob[\texttt{Organic Click|Ad Shown}] \\
        &+ \prob[\texttt{Ad Not Shown}] \\
        &*\prob[\texttt{Organic Click|Ad Not Shown}]
    \end{split}
\end{equation}
Note that \texttt{Ad Shown} and \texttt{Ad Not Shown} are intermediate events calculated from the estimated events. \texttt{Ad Click} is first used to learn the probabilities of \texttt{Ad Shown} and \texttt{Ad Click|Ad Shown} which is then used to estimate \texttt{Search}'s and \texttt{Ad Shown|Search}'s probabilities. Finally, \texttt{Organic Click} is used to learn the probabilities of \texttt{Organic Click|Ad Shown} and \texttt{Organic Click|Ad Not Shown} by modeling the last equation in (\ref{eq:search}).

The BCE loss $\mathcal{L}_b$ is defined on the 2 observed variables \{\texttt{Organic Click}, \texttt{Ad Click}\} while the aggregate loss $\Delta_b$ is computed over all 8 variables - the 5 variables directly estimated by the model, the 2 observed variables and \texttt{Ad Shown}. In $\Delta_b$, the sample average of 3 events \{\texttt{Search}, \texttt{Ad Shown|Search}, \texttt{Organic Click|Ad Shown}\} is obtained from domain knowledge while the sample average of other variables is computed using (\ref{eq:search}). 

Just as in the correctness experiment, the network consists of 5 separate fully connected units one for each of the predicted variables. The architecture of each unit is same as earlier. The dataset of 154982 customers is divided into training, validation and test sets in the proportion $0.55:0.2:0.25$ to enable model selection on the basis of the validation loss.

As mentioned earlier, the results in this case are evaluated by examining the consistency of the results across different initialization. The consistency is quantified as the fraction of samples that are classified likewise in 2 different runs of the model. The classes are obtained by hard limiting the estimated probabilities at $0.5$. This is repeated using each of the 3 training strategies and the consistency results are compared for all the 8 variables.
\setlength{\tabcolsep}{0pt}
\begin{table}[!ht]
\centering
\caption{Comparison of percentage consistency in estimation of all the variables in the Search data using the three training strategies.}
\vspace{-0.5em}
\label{tab:real_data_consistent}
\resizebox{0.48\textwidth}{!}{%
\begin{tabular}{ccccc}
\toprule
 & \texttt{Search} & \texttt{AdShow|Search} & \texttt{Ad Show} & \texttt{OrgClick|Ad}  \\
 & (Unobs) & (Unobs) & (Unobs) & (Unobs)  \\
\midrule
\texttt{BCEL} &   92.15   &   98.43   &   100.0   &  91.69     \\ 
\texttt{AGGL} &   \textbf{99.46}   &   98.29   &   100.0   &   \textbf{96.54}    \\ 
\texttt{SAGG} &  99.23   &   \textbf{98.68}   &   100.0   &   96.48   \\ 
\midrule
& \texttt{OrgClick|NoAd} & \texttt{OrgClick} & \texttt{AdClick|Ad} & \texttt{AdClick} \\
& (Unobs) & (Obs) & (Unobs) & (Obs)  \\
\midrule
\texttt{BCEL} &   98.45   &   99.70   &   \textbf{99.42}   &   100.0     \\ 
\texttt{AGGL} &   \textbf{100.0}   &   \textbf{99.87}   &   91.76   &   100.0    \\ 
\texttt{SAGG} &   99.98   &   99.66   &   90.79   &   100.0    \\ 
\bottomrule
\end{tabular}
}
\end{table}

\vspace{-0.5em}
\section{Results}
% For all the tabular results to follow, the labelling of the experiments is as follows: 
% % To DO: Ayush - This can be moved to the Experiment section
% \begin{itemize} 
%     \item \texttt{BCEL} - with simple BCE loss, %A
%     \item \texttt{POST} - with post-processing, %B
%     \item \texttt{REGL} - with normal regularizer in the loss function, %C
%     \item \texttt{SREG} - with smoothed regularizer in the loss function, %D
%     \item \texttt{PORE} - with smoothed regularizer in the loss function + post-processing. %E
% \end{itemize}

This section contains a discussion on the results obtained for the experiments described in Section \ref{sec:experiments}. For all the experiments involving exponential smoothing we use a fixed value of the smoothing parameter $\alpha=0.8$. Other parameters such as the relative weight $\lambda$ of the aggregate loss $\Delta_b$ are tuned manually with the help of grid search.
\vspace{-0.5em}
\subsection{Results on Synthetic Data}
% The five synthetic data experiments described in Section~\ref{sec:experiments} are conducted on three synthetically generated dataset (see Section~\ref{subsection:syntheticdata}): independent covariates determine the two unobserved events, independent covariates but which feature determine which event is not known, and partial but not full overlap. 

% and therefore the scale factor is not identified correctly with simple BCE loss. MAPE is a good indication of whether the scale is identified correctly or not. The high values of MAPE under \texttt{BCEL} is not surprising. As shown in Section~\ref{sec:model}, the unobserved events are identified up to a scale.

% Looking at the results for the training scenarios \texttt{AGGL} and \texttt{SAGG} when the proposed aggregate loss is used, the value of the MAPE drops significantly as compared to \texttt{BCEL}. The low value of MAPE proves that the scale of the event probabilities is identified correctly using the aggregate loss.
 
% In the \texttt{IND COV UNK} case where the unobserved events are determined by mutually exclusive independent set of features which is unknown during the modeling stage, the proposed approach works well when we add the aggregate loss to the loss function and take exponential smoothing over the mini-batches. This result shows that the approach works well even if we do not know the which subset of the features determines an unobserved event. 

The results in Table~\ref{tab:synthetic_low_prob} show that the proposed approach of using an aggregate loss leads to a large reduction in the estimation error as measured by the MSE and MAPE.
This is for the case when Assumption \textbf{A.3} is not satisfied. 
Hence, the probabilities of the unobserved events are identified only up to a scale with simple BCE loss. This is corroborated by the large values of MAPE under \texttt{BCEL} as MAPE is a good indicator for correct identification of scale. Adding the aggregate loss and smoothing it results in a significant improvement in the performance which can be attributed to correct identification of the scale factor. As indicated in the table, \texttt{AGGL} or \texttt{SAGG} yields the best performance in all the scenarios of synthetic data across the 3 variables. On an average, \texttt{AGGL} provides an improvement of 46\% in MSE and 50\% in MAPE over \texttt{BCEL}. Note that this is when Assumption \textbf{A.3} does not hold.
The natural question is - what happens when Assumption \textbf{A.3} is satisfied and the scale is identified as well.

Table \ref{tab:synthetic_high_prob} shows the results on synthetic data when Assumption \textbf{A.3} holds. 
The results in Table~\ref{tab:synthetic_high_prob} show that providing additional information to the model in the form of aggregate averages for the unobserved event is beneficial even when the probabilities for the unobserved events are identified. The benefit of adding the aggregate loss term is minimal in the unrealistic case where the covariates that determine the two events are independent and known. For the realistic cases \texttt{IND COV UNK} and \texttt{PAR OV UNK}, the MSE and MAPE errors are reduced on average by 33\% and 28\%, respectively. This is significant improvement over the baseline, albeit much smaller than the improvement in results in Table~\ref{tab:synthetic_low_prob} when the unobserved probabilities are identified only up to a scale.
Note that the proposed model performs significantly better than \texttt{BCEL} in the \texttt{COM OV} case also, where Assumption \textbf{A.1} does not hold.
% \texttt{PAR OV}  the result when the independence assumption (\textbf{A.3}) does not hold. 

% \shiv{In this case, the F-Score for $Y_1$ and $Y_2$ improves, whereas the F-Score for $Y$ goes down.} Again, the smooth version of the regularizer generates the best results. 
The synthetic data experiments support the theoretical results in Theorem~\ref{theorem:identuptoscale} and Corollary~\ref{corollary:scaleident} as well as validate the proposed approach to identify the scale. Next, the proposed approach is applied on the Customer Behavior Data.
\vspace{-0.5em}
\subsection{Results on Customer Behavior Data}
As noted in Section~\ref{subsec:exprealdata}, there are two set of experiments on the real data. The validation done using the synthetic data can only be performed on Email data. For Search data, we test for consistency across multiple runs to validate identification.
\subsubsection{Correctness Experiment}
% \begin{itemize}
    % \item Table \ref{tab:correctness_allobserved} contains the precision, recall and F-score results for \texttt{Email Sent}, \texttt{Email Opened} and \texttt{Email Clicked} when all the events are visible to the model during training.
    % \item Table \ref{tab:correctness_hidden} reports the same 3 metrics for \texttt{Email Sent} and \texttt{Email Opened} when the 2 events respectively are hidden from the model during training.
    % \item Table \ref{tab:correctness_openedunobserved} contains the results of how well \texttt{Email Opened} event is recovered by the model when it is assumed to be unobserved during training.
% \end{itemize}

Table \ref{tab:real_data_correct} contains the validation results on the Email data in terms of MSE and MAPE. As seen in the results on synthetic data, the proposed method of adding the aggregate loss term performs much better in identifying the unobserved event probabilities. It performs well for the observed events as well, but the improvement is much larger in case of unobserved events. The methods \texttt{AGGL} and \texttt{SAGG} help in especially reducing MAPE which is a good indication that the approach is successful in identifying the scale factor correctly. On an average over the 5 variables, \texttt{AGGL} reduces the MSE and MAPE errors by 17\% and 46\% respectively as compared to \texttt{BCEL}.
\makeatletter{\renewcommand*{\@makefnmark}{$\ddag$}
\footnotetext{MSE values are reported as percentages, i.e. multiplied by $10^2$.}\makeatother}

\subsubsection{Consistency Experiment}
Table \ref{tab:real_data_consistent} illustrates the identification results in the Search data as quantified by percentage consistency across 2 different random runs of the model.
As can be seen from the table, the estimates of the events do not differ much across different initializations. Interestingly, the results across runs become better when the aggregate loss term is used, except for the event \{\texttt{Ad Click|Ad Shown}\}. The improvement in the performance is significant in \{\texttt{Search}\} and \{\texttt{Organic Search Click|Ad Shown}\}. These results show the benefit of including the aggregate loss term in the loss function.

%%% Here, put the results for unbalanced dataset using the weighted (w0 and w1) loss function
 
% Table \ref{tab:consistency} compares the consistency results for two unobserved events each in the Search and the Email use-case. The comparison is for the following experiments:
%     \begin{itemize}
%         \item simple BCE loss function
%         \item Error minimization using post-processing
%         \item Normal regularizer augmentation in loss function
%         \item Smoothed regularizer augmentation in loss function
%         \item Smoothed regularizer + post-processing 
%     \end{itemize}

% \input{tables/RealDataConsistency.tex}
%% Now included in the real data experiments section just before the beginning of Correctness experiments. Done this for better positioning of the tables in the pdf.

%vspace{-0.5em}
\section{Conclusion and Future Work}
This paper provided an approach that can be used to identify unobserved events when a composite event, that is based on multiple unobserved or observed events, is observed. 
We have provided theoretical proof of identification under mild conditions.
The empirical approach allows inference on events in customer-brand interaction setting.
The solution can be applied in marketing touch attribution setting or for estimating a model that can be used for running simulations on marketing actions.

There is very little work in this area. 
Nonetheless, it is an important problem due to marketers spending more and more money on earned and paid channels. 
The research in this direction has a potential of improving targeting strategies on the channels that are not owned by a brand without compromising the privacy by using aggregate data. 
There are a few future directions that this research can take. 
First, the identification can be proved assuming that a consistent estimator for the probability of the observed composite event is available.
Second, the mini-batch smoothing approach used in this paper can be further refined.
Second, a segment level reporting of the aggregate statistics might further improve the results.

\bibliographystyle{ACM-Reference-Format}
\bibliography{main}


\begin{thebibliography}{00}

%%% the Web address.

\ifx \showCODEN    \undefined \def \showCODEN     #1{\unskip}     \fi
\ifx \showDOI      \undefined \def \showDOI       #1{#1}\fi
\ifx \showISBNx    \undefined \def \showISBNx     #1{\unskip}     \fi
\ifx \showISBNxiii \undefined \def \showISBNxiii  #1{\unskip}     \fi
\ifx \showISSN     \undefined \def \showISSN      #1{\unskip}     \fi
\ifx \showLCCN     \undefined \def \showLCCN      #1{\unskip}     \fi
\ifx \shownote     \undefined \def \shownote      #1{#1}          \fi
\ifx \showarticletitle \undefined \def \showarticletitle #1{#1}   \fi
\ifx \showURL      \undefined \def \showURL       {\relax}        \fi
% invisible to TeX
\providecommand\bibfield[2]{#2}
\providecommand\bibinfo[2]{#2}
\providecommand\natexlab[1]{#1}
\providecommand\showeprint[2][]{arXiv:#2}

\bibitem[\protect\citeauthoryear{Abhishek, Despotakis, and Ravi}{Abhishek
  et~al\mbox{.}}{2017}]%
        {abhishek2017multi}
\bibfield{author}{\bibinfo{person}{Vibhanshu Abhishek},
  \bibinfo{person}{Stylianos Despotakis}, {and} \bibinfo{person}{R Ravi}.}
  \bibinfo{year}{2017}\natexlab{}.
\newblock \showarticletitle{Multi-channel attribution: The blind spot of online
  advertising}.
\newblock \bibinfo{journal}{{\em Available at SSRN 2959778\/}}
  (\bibinfo{year}{2017}).
\newblock


\bibitem[\protect\citeauthoryear{Bertsimas, Pawlowski, and Zhuo}{Bertsimas
  et~al\mbox{.}}{2017}]%
        {bertsimas2017predictive}
\bibfield{author}{\bibinfo{person}{Dimitris Bertsimas}, \bibinfo{person}{Colin
  Pawlowski}, {and} \bibinfo{person}{Ying~Daisy Zhuo}.}
  \bibinfo{year}{2017}\natexlab{}.
\newblock \showarticletitle{From predictive methods to missing data imputation:
  an optimization approach}.
\newblock \bibinfo{journal}{{\em The Journal of Machine Learning Research\/}}
  \bibinfo{volume}{18}, \bibinfo{number}{1} (\bibinfo{year}{2017}),
  \bibinfo{pages}{7133--7171}.
\newblock


\bibitem[\protect\citeauthoryear{Besbes, Desir, Goyal, Iyengar, and
  Singal}{Besbes et~al\mbox{.}}{2019}]%
        {besbes2019shapley}
\bibfield{author}{\bibinfo{person}{Omar Besbes}, \bibinfo{person}{Antoine
  Desir}, \bibinfo{person}{Vineet Goyal}, \bibinfo{person}{Garud Iyengar},
  {and} \bibinfo{person}{Raghav Singal}.} \bibinfo{year}{2019}\natexlab{}.
\newblock \showarticletitle{Shapley Meets Uniform: An Axiomatic Framework for
  Attribution in Online Advertising}. In \bibinfo{booktitle}{{\em The World
  Wide Web Conference}}. ACM, \bibinfo{pages}{1713--1723}.
\newblock


\bibitem[\protect\citeauthoryear{Bianchi}{Bianchi}{2018}]%
        {Bianchi2018}
\bibfield{author}{\bibinfo{person}{Kerry Bianchi}.}
  \bibinfo{year}{2018}\natexlab{}.
\newblock \bibinfo{title}{Can Omnichannel Marketing Exist in a World of Walled
  Gardens?}
\newblock   (\bibinfo{year}{2018}).
\newblock
\showURL{%
\url{https://www.adweek.com/digital/can-omnichannel-marketing-exist-in-a-world-of-walled-gardens/}}


\bibitem[\protect\citeauthoryear{Centre}{Centre}{2020}]%
        {fb2020}
\bibfield{author}{\bibinfo{person}{Business~Help Centre}.}
  \bibinfo{year}{2020}\natexlab{}.
\newblock \bibinfo{title}{About using Ads Manager to understand ad
  performance}.
\newblock   (\bibinfo{year}{2020}).
\newblock
\showURL{%
\url{https://en-gb.facebook.com/business/help/510910008975690?id=369013183583436}}


\bibitem[\protect\citeauthoryear{Danaher and van Heerde}{Danaher and van
  Heerde}{2018}]%
        {danaher2018delusion}
\bibfield{author}{\bibinfo{person}{Peter~J Danaher} {and}
  \bibinfo{person}{Harald~J van Heerde}.} \bibinfo{year}{2018}\natexlab{}.
\newblock \showarticletitle{Delusion in attribution: Caveats in using
  attribution for multimedia budget allocation}.
\newblock \bibinfo{journal}{{\em Journal of Marketing Research\/}}
  \bibinfo{volume}{55}, \bibinfo{number}{5} (\bibinfo{year}{2018}),
  \bibinfo{pages}{667--685}.
\newblock


\bibitem[\protect\citeauthoryear{Ioffe}{Ioffe}{2017}]%
        {Ioffe2017BatchRenormalization}
\bibfield{author}{\bibinfo{person}{Sergey Ioffe}.}
  \bibinfo{year}{2017}\natexlab{}.
\newblock \showarticletitle{Batch Renormalization: Towards Reducing Minibatch
  Dependence in Batch-Normalized Models}. In \bibinfo{booktitle}{{\em
  Proceedings of the 31st International Conference on Neural Information
  Processing Systems}} {\em (\bibinfo{series}{NIPS’17})}.
  \bibinfo{publisher}{Curran Associates Inc.}, \bibinfo{address}{Red Hook, NY,
  USA}, \bibinfo{pages}{1942–1950}.
\newblock
\showISBNx{9781510860964}


\bibitem[\protect\citeauthoryear{Ji and Wang}{Ji and Wang}{2017}]%
        {ji2017additional}
\bibfield{author}{\bibinfo{person}{Wendi Ji} {and} \bibinfo{person}{Xiaoling
  Wang}.} \bibinfo{year}{2017}\natexlab{}.
\newblock \showarticletitle{Additional multi-touch attribution for online
  advertising}. In \bibinfo{booktitle}{{\em Thirty-First AAAI Conference on
  Artificial Intelligence}}.
\newblock


\bibitem[\protect\citeauthoryear{Kakalej{\v{c}}{\'\i}k, Bucko, Resende, and
  Ferencova}{Kakalej{\v{c}}{\'\i}k et~al\mbox{.}}{2018}]%
        {kakalejvcik2018multichannel}
\bibfield{author}{\bibinfo{person}{Luk{\'a}{\v{s}} Kakalej{\v{c}}{\'\i}k},
  \bibinfo{person}{Jozef Bucko}, \bibinfo{person}{Paulo~AA Resende}, {and}
  \bibinfo{person}{Martina Ferencova}.} \bibinfo{year}{2018}\natexlab{}.
\newblock \showarticletitle{Multichannel Marketing Attribution Using Markov
  Chains}.
\newblock \bibinfo{journal}{{\em Journal of Applied Management and
  Investments\/}} \bibinfo{volume}{7}, \bibinfo{number}{1}
  (\bibinfo{year}{2018}), \bibinfo{pages}{49--60}.
\newblock


\bibitem[\protect\citeauthoryear{Ren, Fang, Zhang, Liu, Li, Zhang, Yu, and
  Wang}{Ren et~al\mbox{.}}{2018}]%
        {Ren2018DualAttention}
\bibfield{author}{\bibinfo{person}{Kan Ren}, \bibinfo{person}{Yuchen Fang},
  \bibinfo{person}{Weinan Zhang}, \bibinfo{person}{Shuhao Liu},
  \bibinfo{person}{Jiajun Li}, \bibinfo{person}{Ya Zhang},
  \bibinfo{person}{Yong Yu}, {and} \bibinfo{person}{Jun Wang}.}
  \bibinfo{year}{2018}\natexlab{}.
\newblock \showarticletitle{Learning Multi-touch Conversion Attribution with
  Dual-attention Mechanisms for Online Advertising}. In
  \bibinfo{booktitle}{{\em Proceedings of the 27th ACM International Conference
  on Information and Knowledge Management}} {\em (\bibinfo{series}{CIKM '18})}.
  \bibinfo{publisher}{ACM}, \bibinfo{address}{New York, NY, USA},
  \bibinfo{pages}{1433--1442}.
\newblock
\showISBNx{978-1-4503-6014-2}
\showDOI{%
\url{https://doi.org/10.1145/3269206.3271677}}


\bibitem[\protect\citeauthoryear{Saini, Dhamnani, Aakash, Ibrahim, and
  Chavan}{Saini et~al\mbox{.}}{2019}]%
        {Saini2019MTE}
\bibfield{author}{\bibinfo{person}{Shiv~Kumar Saini}, \bibinfo{person}{Sunny
  Dhamnani}, \bibinfo{person}{Aakash}, \bibinfo{person}{Akil~Arif Ibrahim},
  {and} \bibinfo{person}{Prithviraj Chavan}.} \bibinfo{year}{2019}\natexlab{}.
\newblock \showarticletitle{Multiple Treatment Effect Estimation Using Deep
  Generative Model with Task Embedding}. In \bibinfo{booktitle}{{\em The World
  Wide Web Conference}} {\em (\bibinfo{series}{WWW '19})}.
  \bibinfo{publisher}{ACM}, \bibinfo{address}{New York, NY, USA},
  \bibinfo{pages}{1601--1611}.
\newblock
\showISBNx{978-1-4503-6674-8}
\showDOI{%
\url{https://doi.org/10.1145/3308558.3313744}}


\bibitem[\protect\citeauthoryear{Schuessler}{Schuessler}{1999}]%
        {Schuessler10578}
\bibfield{author}{\bibinfo{person}{Alexander~A. Schuessler}.}
  \bibinfo{year}{1999}\natexlab{}.
\newblock \showarticletitle{Ecological inference}.
\newblock \bibinfo{journal}{{\em Proceedings of the National Academy of
  Sciences\/}} \bibinfo{volume}{96}, \bibinfo{number}{19}
  (\bibinfo{year}{1999}), \bibinfo{pages}{10578--10581}.
\newblock
\showISSN{0027-8424}
\showDOI{%
\url{https://doi.org/10.1073/pnas.96.19.10578}}
\showeprint{https://www.pnas.org/content/96/19/10578.full.pdf}


\bibitem[\protect\citeauthoryear{Shalit, Johansson, and Sontag}{Shalit
  et~al\mbox{.}}{2017}]%
        {Shalit2017EstimatingIT}
\bibfield{author}{\bibinfo{person}{Uri Shalit}, \bibinfo{person}{Fredrik~D.
  Johansson}, {and} \bibinfo{person}{David~A Sontag}.}
  \bibinfo{year}{2017}\natexlab{}.
\newblock \showarticletitle{Estimating individual treatment effect:
  generalization bounds and algorithms}. In \bibinfo{booktitle}{{\em ICML}}.
\newblock


\bibitem[\protect\citeauthoryear{{\'S}mieja, Struski, Tabor, Zieli{\'n}ski, and
  Spurek}{{\'S}mieja et~al\mbox{.}}{2018}]%
        {smieja2018processing}
\bibfield{author}{\bibinfo{person}{Marek {\'S}mieja},
  \bibinfo{person}{{\L}ukasz Struski}, \bibinfo{person}{Jacek Tabor},
  \bibinfo{person}{Bartosz Zieli{\'n}ski}, {and} \bibinfo{person}{Przemys{\l}aw
  Spurek}.} \bibinfo{year}{2018}\natexlab{}.
\newblock \showarticletitle{Processing of missing data by neural networks}. In
  \bibinfo{booktitle}{{\em Advances in Neural Information Processing Systems}}.
  \bibinfo{pages}{2719--2729}.
\newblock


\bibitem[\protect\citeauthoryear{Zhao, Mahboobi, and Bagheri}{Zhao
  et~al\mbox{.}}{2019}]%
        {zhao2019revenue}
\bibfield{author}{\bibinfo{person}{Kaifeng Zhao}, \bibinfo{person}{Seyed~Hanif
  Mahboobi}, {and} \bibinfo{person}{Saeed~R Bagheri}.}
  \bibinfo{year}{2019}\natexlab{}.
\newblock \showarticletitle{Revenue-based attribution modeling for online
  advertising}.
\newblock \bibinfo{journal}{{\em International Journal of Market Research\/}}
  \bibinfo{volume}{61}, \bibinfo{number}{2} (\bibinfo{year}{2019}),
  \bibinfo{pages}{195--209}.
\newblock


\end{thebibliography}

\end{document}